\pdfoutput=1
\documentclass{article} 
\usepackage{iclr2027_conference,times}

\usepackage[utf8]{inputenc}
\usepackage[T1]{fontenc}
\usepackage{hyperref}
\usepackage{url}
\usepackage{booktabs}
\usepackage{amsfonts}
\usepackage{amsmath}
\usepackage{amssymb}
\usepackage{amsthm}
\usepackage{nicefrac}
\usepackage{microtype}
\usepackage{xcolor}
\usepackage{graphicx}
\usepackage{algorithm}
\usepackage{algpseudocode}
\usepackage{enumitem}
\usepackage[capitalize,noabbrev]{cleveref}

\theoremstyle{plain}
\newtheorem{theorem}{Theorem}
\newtheorem{lemma}{Lemma}
\newtheorem{corollary}{Corollary}
\newtheorem{proposition}{Proposition}

\theoremstyle{definition}
\newtheorem{definition}{Definition}
\newtheorem{remark}{Remark}

\crefname{theorem}{Theorem}{Theorems}
\crefname{lemma}{Lemma}{Lemmas}
\crefname{corollary}{Corollary}{Corollaries}
\crefname{proposition}{Proposition}{Propositions}
\crefname{definition}{Definition}{Definitions}
\crefname{remark}{Remark}{Remarks}
\crefname{algorithm}{Algorithm}{Algorithms}
\crefname{section}{Section}{Sections}
\crefname{table}{Table}{Tables}
\crefname{figure}{Figure}{Figures}

\DeclareMathOperator{\clip}{clip}
\DeclareMathOperator{\Ortho}{Ortho}
\DeclareMathOperator{\NS}{NS}
\DeclareMathOperator{\RMS}{RMS}
\DeclareMathOperator{\Var}{Var}
\DeclareMathOperator{\corr}{corr}
\DeclareMathOperator{\rank}{rank}

\newcommand{\R}{\mathbb{R}}
\newcommand{\E}{\mathbb{E}}

\title{OrScale: Orthogonalized Optimization with\\ Layer-Wise Trust-Ratio Scaling}

\author{Yuxuan Lou\\
National University of Singapore\\
\texttt{yuxuanlou@u.nus.edu}
\And
Yang You\\
National University of Singapore\\
\texttt{yangyou@nus.edu.sg}}

\iclrfinalcopy

\begin{document}

\maketitle

\begingroup
\renewcommand\thefootnote{}%
\footnotetext{Code, per-result configuration files, and raw training logs are open-sourced at \url{https://github.com/NUS-HPC-AI-Lab/OrScale}.}%
\endgroup

\begin{abstract}
Muon fixes the \emph{direction} of every matrix-valued update at the polar factor of its momentum, while each layer's step \emph{magnitude} is addressed only by a static shape correction. We derive a dynamic per-layer scalar by adapting the LARS/LAMB trust-ratio principle to the orthogonalized setting, where the standard denominator candidates---the raw momentum norm or the polar-factor norm---either live in the wrong unit space or carry no update-scale information. The resulting method, \emph{OrScale}, uses the norm of the parameter-space direction actually applied and anchors each layer's ratio at one via a per-layer calibration, so that the Moonlight recipe (tuned for AdamW, shared with Muon via RMS matching) transfers with \emph{no additional sweep}; a component ablation confirms each design choice is individually load-bearing. Theoretically, OrScale retains a nuclear-norm $O(1/\sqrt{T})$ convergence rate for any clipped multiplier and achieves a strict layer-adaptive descent gain $\kappa_{\mathrm{eff}}>1$ under two conditions estimable from standard training diagnostics---a bound that predicts the gain should \emph{grow with architectural heterogeneity}. Experiments confirm the prediction: with every hyperparameter inherited verbatim from the Moonlight recipe, OrScale matches or beats Muon+Moonlight across dense 125M--1.1B FineWeb-Edu pre-training, and on a 16B-A3B mixture-of-experts model---where the logged trust ratios separate cleanly by layer class---the gap widens by an order of magnitude to $0.130$ nats ($3.8\%$ relative) at parity wall-clock cost.
\end{abstract}

\section{Introduction}
\label{sec:intro}

Muon has rapidly moved from a compact optimizer idea to a practical ingredient of large-scale training. Its core step takes a matrix-valued momentum, applies a few Newton--Schulz iterations to approximate the polar factor, and subtracts this spectrally flattened direction from the weights; the update can be read as steepest descent under the spectral norm~\citep{bernstein2024modular}, and production systems---Moonlight~\citep{liu2025moonlight}, Kimi K2.5~\citep{kimi2026k25visual}, DeepSeek-V4~\citep{deepseekv42026}---and the modded-nanogpt records~\citep{jordan2024muon,jordan2024moddednanogpt} have made orthogonalized updates a fixture of current large-scale practice.

Orthogonalization settles the direction of each layer's update but leaves its \emph{magnitude} under-specified: after the polar factor, every matrix layer moves by the same global learning rate, even though attention projections, MLP blocks, routers, and experts see very different gradient statistics and curvature. Moonlight corrects the predictable part with a static per-shape factor $0.2\sqrt{\max(m,n)}$~\citep{liu2025moonlight}, but a static factor cannot track how each layer's update scale evolves \emph{during} training. For SGD and Adam, LARS and LAMB addressed precisely this problem by introducing a \emph{trust ratio}---the weight norm divided by the norm of the direction about to be subtracted---that scales each layer's step so that relative step sizes stay comparable across layers, enabling stable large-batch training~\citep{you2017lars,you2020lamb}. Can the same idea be applied to an orthogonalized optimizer, and if so, what should the ratio look like when the update direction has been spectrally flattened?

With the orthogonalized direction kept untouched, the remaining choice---one positive scalar per layer per step---can be derived rather than searched. Following the LARS/LAMB template, the quantity a layer-wise scalar should regulate is the \emph{relative} step $\|\Delta W_\ell\|_F/\|W_\ell\|_F$; holding it near a common target forces a trust ratio with numerator $\|W_\ell\|_F$ and denominator the norm of whatever the optimizer is about to subtract. After orthogonalization, however, that denominator stops being a bookkeeping choice: the direction actually applied is the shape-scaled polar factor together with weight decay, not the raw momentum, so the denominator should be the Frobenius norm of that parameter-space direction rather than a gradient-space surrogate. A ratio built this way would still, at initialization, silently shrink an already-tuned step by a layer- and width-dependent factor; a one-time per-layer calibration therefore anchors every ratio at exactly one, making the first update coincide with Muon+Moonlight's. And because a weight-norm numerator can feed its own growth, weight decay is coupled \emph{inside} the scaled update, giving the ratio a finite ceiling. Together these design choices define OrScale (\cref{sec:method}), with Muon and Muon+Moonlight recovered as the limit $r_{\min}{=}r_{\max}{=}1$; a component ablation shows each choice is individually load-bearing, each deletion producing a distinct, logged failure (\cref{sec:req-ablation}).

\begin{figure}[t]
\centering
\includegraphics[width=\linewidth]{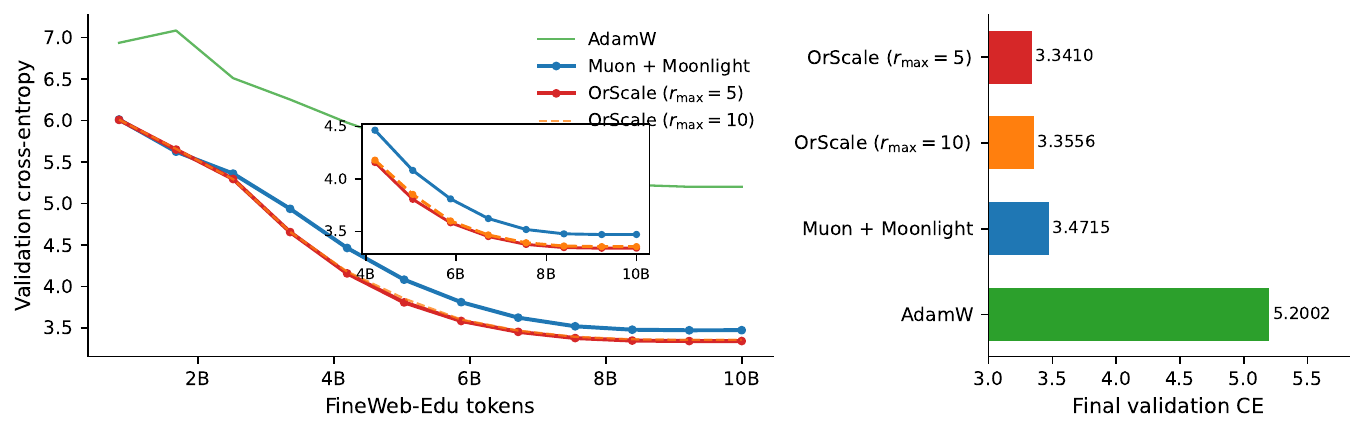}
\caption{\textbf{Moonlight-16B-A3B mixture-of-experts pre-training on 10B FineWeb-Edu tokens}, every hyperparameter inherited verbatim from the Moonlight recipe (LR $4.2\times10^{-4}$ shared by all optimizers; OrScale never swept; \cref{sec:protocol}). \emph{Left:} validation cross-entropy (inset: final 6B tokens). \emph{Right:} final values. OrScale improves on Muon+Moonlight by $0.130$ nats ($3.8\%$ relative)---an order of magnitude beyond the $+0.011$-nat dense gap at 1.1B---as the heterogeneity condition of \cref{sec:theory} predicts (tested in \cref{sec:moe-heterogeneity}).}
\label{fig:moe-headline}
\end{figure}

The calibration anchor doubles as our evaluation protocol: the learning rate and weight decay of the Moonlight recipe---tuned for AdamW, shared with Muon via RMS matching---carry over to OrScale unchanged, so \emph{every OrScale language-model number in this paper, 125M dense through a 16B-A3B mixture-of-experts model, runs at the untouched recipe learning rate; OrScale was never swept at any language-model scale}. The theory then says when the added machinery should pay: the gain $\kappa_{\mathrm{eff}}$ over Muon exceeds $1$ whenever the logged trust ratios are heterogeneous across layers and aligned with the layer-wise signal-to-smoothness ratio---two conditions estimable from training diagnostics, both expected to strengthen with architectural heterogeneity. That makes mixture-of-experts models a natural stress test, and the experiment agrees (\cref{fig:moe-headline}).

\textbf{Contributions.}
\begin{itemize}[leftmargin=1.2em,itemsep=2pt,topsep=2pt]
  \item \textbf{A trust ratio from first principles.} A forward derivation of Muon's layer-wise trust ratio from the LARS/LAMB template, adapted to the orthogonalized setting; a parameter-light algorithm whose limits recover Muon and Muon+Moonlight; and a component ablation in which each design-choice deletion fails in its own recognizable way (\cref{sec:method}).
  \item \textbf{Zero-sweep transfer as a protocol.} The calibration anchor lets a Moonlight/AdamW-tuned recipe transfer with no additional sweep---an ability we then lean on at every language-model scale we train (\cref{sec:protocol}).
  \item \textbf{Theory conditioned on measurable statistics.} A nuclear-norm $O(1/\sqrt T)$ rate for any clipped multiplier, a strict gain $\kappa_{\mathrm{eff}}>1$ under two conditions checkable from logged trust ratios and standard per-layer diagnostics rather than oracle smoothness constants alone, and a collapse theorem for unit-inconsistent ratios (\cref{sec:theory}).
  \item \textbf{Validation from 6.5M to 16B parameters.} OrScale ranks first on CIFAR-10/DavidNet, matches or beats Muon+Moonlight over dense 125M--1.1B FineWeb-Edu pre-training, and widens the gap to $+0.130$ nats on 16B-A3B MoE at wall-clock parity (\cref{sec:experiments}).
\end{itemize}

\section{Background and positioning}
\label{sec:related}

\textbf{Muon and orthogonalized updates.}
Muon~\citep{jordan2024muon} updates each matrix-valued parameter $W_\ell\in\R^{m_\ell\times n_\ell}$ by $W_{\ell,t+1}=W_{\ell,t}-\eta_t Q_{\ell,t}$ (optional weight decay aside), where $Q_{\ell,t}=\NS_k(\widetilde M_{\ell,t})$ is a $k$-step Newton--Schulz approximation to the polar factor $\Ortho(\widetilde M_{\ell,t})$ of the Nesterov-lookahead gradient momentum. If $Q_A=\Ortho(A)$ then $\langle A,Q_A\rangle=\|A\|_*$ and $\|Q_A\|_F=\sqrt{\rank A}$, making orthogonalization a nuclear-norm/spectral primitive rather than an entrywise adaptive update~\citep{bernstein2024modular}. Moonlight adds the static shape factor $0.2\sqrt{\max(m,n)}$, equalizing the per-element RMS of $Q$ across shapes and pinning it to AdamW's typical update RMS so that AdamW-tuned learning rates carry over~\citep{liu2025moonlight}; MuonClip~\citep{kimi2025muonclip} and AdaMuon~\citep{adamuon2025} extend Muon along other axes. OrScale keeps the Muon direction untouched and addresses the complementary question of the per-layer scalar magnitude over training time.

\textbf{Trust-ratio optimizers.}
LARS~\citep{you2017lars} scales each layer's SGD step by $\|W_\ell\|_F/\|G_\ell+\lambda W_\ell\|_F$---weight norm over the norm of the exact direction SGD is about to subtract---holding every layer's \emph{relative} step near a common target; LAMB~\citep{you2020lamb} applies the same template to the Adam-normalized direction and enabled large-batch BERT training~\citep{devlin2018bert,goyal2017accuratelarge}. The template silently relies on a unit convention---numerator and denominator are both norms of parameter-space objects---that Muon's pipeline breaks at both ends: the post-orthogonalization norm $\|Q_\ell\|_F$ is a shape constant, and the momentum norm $\|\widetilde M_\ell\|_F$ has gradient units. \Cref{sec:derivation} turns this unit analysis into the denominator that restores the convention (\cref{def:principle}); \cref{sec:req-ablation} shows what happens when it is dropped.

\textbf{Static corrections and width transfer.}
Moonlight's shape factor and $\mu$P-style parameterization~\citep{yang2022tensorprogramsvtuning,yang2023tensorprogramsvifeature} handle predictable shape and width effects at initialization; a trust ratio handles training-time drift, and OrScale's LM configuration composes the two---reproducing the Moonlight step exactly at $t{=}0$, adapting thereafter. Preconditioner methods such as Sophia~\citep{liu2023sophia} and SOAP~\citep{vyas2025soap} modify the direction and are orthogonal to the magnitude question here.

\textbf{Positioning within Muon convergence theory.}
A rapidly growing line of work analyzes Muon and its relatives: smoothness-based rates for the exact and inexact polar step~\citep{li2025muonnote}, non-Euclidean trust-region and LMO views~\citep{kovalev2025orthogonalization,pethick2025scion}, and layer-wise $(L^0,L^1)$ smoothness frameworks with per-layer stepsizes such as Gluon~\citep{riabinin2025gluon}. These analyses obtain sharper constants than ours, but condition their layer-adaptive conclusions on \emph{oracle} smoothness constants unobservable during training. Ours trades sharpness for verifiability: the gain condition of \cref{thm:gain} involves the \emph{logged, clipped trust ratios} $\hat r_{\ell,t}$---quantities OrScale computes anyway---and per-layer statistics estimable from training diagnostics, so whether the gain regime holds is checkable on any run; we check it at 16B in \cref{sec:moe-heterogeneity}, and \cref{app:compare} tabulates the guarantee structures.

\section{Deriving the OrScale trust ratio}
\label{sec:method}

\begin{figure}[t]
\centering
\includegraphics[width=\linewidth]{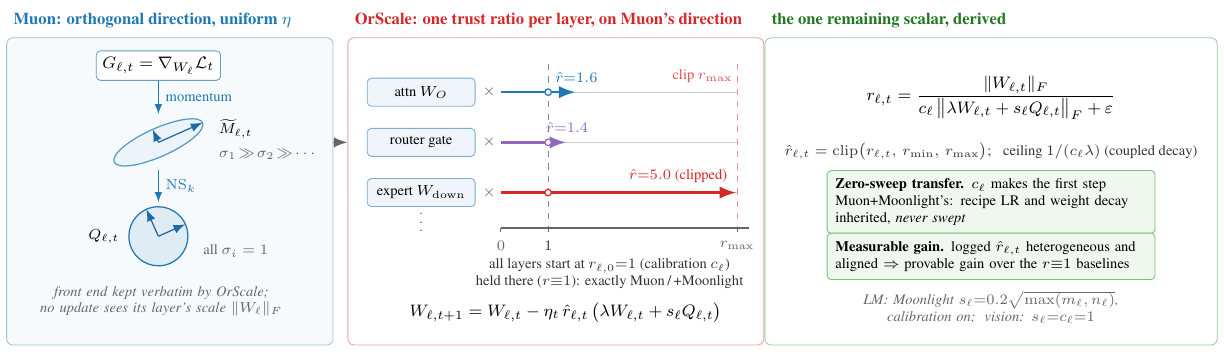}
\caption{\textbf{OrScale at a glance.} \emph{Left:} Muon fixes each update's direction at the polar factor of its momentum, leaving the per-layer magnitude to one layer-blind global $\eta$. \emph{Center:} OrScale supplies that magnitude---one clipped trust ratio $\hat r_{\ell,t}$ per matrix layer, on the unchanged direction. Arrow lengths are ratios logged in the 16B MoE run at 1.7B tokens, already spread $1.4$--$5.0$ across layer classes with the expert projection at the clip ceiling (\cref{sec:moe}; per-class colors and full trajectories in \cref{fig:moe-trust}); the calibration $c_\ell$ starts every layer at the $r{=}1$ guide (\cref{def:calib}), where $r_{\min}{=}r_{\max}{=}1$ recovers Muon and Muon+Moonlight exactly. \emph{Right:} the ratio---relative-step numerator over applied-direction denominator (\cref{def:principle}), bounded through the coupled decay (\cref{lem:ceiling})---with its gain condition measurable from the logged ratios (\cref{sec:theory}).}
\label{fig:pipeline}
\end{figure}

\textbf{Notation.} $W_\ell\in\R^{m_\ell\times n_\ell}$ is a matrix parameter, $p_\ell=\min(m_\ell,n_\ell)$, $\|\cdot\|_F$ and $\|\cdot\|_*$ are the Frobenius and nuclear norms, $\Ortho(\cdot)$ the polar factor, $\NS_k$ its $k$-step Newton--Schulz approximation, $\widetilde M_{\ell,t}$ the Nesterov-lookahead momentum, $\clip(\cdot,r_{\min},r_{\max})$ clamps to an interval, and $\varepsilon$ is a small floor. Every Muon-family optimizer in this paper shares the front end
\begin{equation}
M_{\ell,t} = \mu\,M_{\ell,t-1}+G_{\ell,t},\qquad
\widetilde M_{\ell,t} = \mu\,M_{\ell,t}+G_{\ell,t},\qquad
Q_{\ell,t} = \NS_k(\widetilde M_{\ell,t}),
\label{eq:frontend}
\end{equation}
the choice used by every reported large-scale Muon system~\citep{liu2025moonlight,jordan2024muon}; non-matrix parameters (biases, norm gains, embeddings when desired) use an AdamW fallback~\citep{loshchilov2019adamw}. \Cref{fig:pipeline} previews the derivation and the update it produces.

\subsection{The derivation}
\label{sec:derivation}

Muon's value lies in its direction: the polar factor equalizes singular values and yields the descent identity $\langle A,Q_A\rangle=\|A\|_*$ on which its analysis rests. We therefore keep the direction and adapt magnitude only: once the direction of the applied step---call it $D_{\ell,t}$---is fixed, the most general layer-wise modification is $W_{\ell,t+1}=W_{\ell,t}-\eta_t\,\rho_{\ell,t}\,D_{\ell,t}$, one positive scalar $\rho_{\ell,t}$ per layer per step. This scalar should regulate the relative step $\|\Delta W_\ell\|_F/\|W_\ell\|_F$, whose drift across layers is precisely what LARS observed to make layers stall or destabilize~\citep{you2017lars}; holding $\|\eta\rho_\ell D_\ell\|_F/\|W_\ell\|_F$ near a common target gives $\rho_\ell\propto\|W_\ell\|_F/\|D_\ell\|_F$---a trust ratio with numerator $\|W_\ell\|_F$ and denominator the norm of the applied direction. The choice of denominator is where Muon parts ways with SGD and Adam: the ratio should measure what is actually subtracted, and for a Muon step with shape factor $s_\ell$ and coupled weight decay the direction is $D_{\ell,t}=\lambda W_{\ell,t}+s_\ell Q_{\ell,t}$.

\begin{definition}[Real-update-direction trust ratio]
\label{def:principle}
Let $D_{\ell,t}$ be the parameter-space direction the optimizer is about to subtract from $W_{\ell,t}$ (before the global learning rate and the trust-ratio scalar). The trust ratio is
$r_{\ell,t} = \|W_{\ell,t}\|_F/\big(\|D_{\ell,t}\|_F+\varepsilon\big)$.
\end{definition}

For SGD, $D=G+\lambda W$ and \cref{def:principle} \emph{is} LARS. After orthogonalization, however, the choice of denominator acquires content, and the reason is units: the pre-orthogonalization momentum $\widetilde M_\ell$ lives in gradient units---its norm set by the loss scale and the noise level---whereas the direction Muon actually subtracts lives in parameter units, so dividing the weight norm by the momentum norm measures no relative step at all, but a dimensionally mixed quantity that can sit orders of magnitude from one. The post-orthogonalization norm fails in the opposite way: $\|Q_\ell\|_F=\sqrt{\rank Q_\ell}$ is a shape constant carrying no update-scale information whatsoever (\cref{lem:polar}). \Cref{def:principle} restores the parameter-space convention on both sides of the ratio.

Two further considerations make the ratio deployable. First, adopting it must not invalidate an already-tuned learning rate---the Moonlight line's main practical asset is that Muon's update RMS is matched to AdamW's, so one recipe serves both~\citep{liu2025moonlight}---yet the raw ratio of \cref{def:principle} would break exactly that: at initialization $\|W_{\ell,0}\|_F/\|s_\ell Q_{\ell,0}\|_F\approx\RMS(W_{\ell,0})/(0.2)\propto 1/\sqrt{\mathrm{fan\mbox{-}in}}$ sits between $0.11$ and $0.26$ across the layers of our 125M model, shrinking the effective step several-fold, by a different factor for every layer and width. The fix is a one-time anchor:

\begin{definition}[Lazy per-layer calibration]
\label{def:calib}
At the first step at which $W_\ell$ receives a non-zero update---write $0$ for that step's index in what follows---store the fp32 scalar
$c_{\ell} = \|W_{\ell,0}\|_F/\big(\|\lambda W_{\ell,0}+s_\ell Q_{\ell,0}\|_F+\varepsilon\big)$,
and use the calibrated ratio $r_{\ell,t} = \|W_{\ell,t}\|_F/\big(c_{\ell}\,\|\lambda W_{\ell,t}+s_\ell Q_{\ell,t}\|_F+\varepsilon\big)$ thereafter.
\end{definition}

The constant records each layer's ratio at initialization, so that the trust ratio thereafter measures \emph{growth relative to initialization} rather than absolute scale. By construction $r_{\ell,0}=1$ (to within the $\varepsilon$ floor) for every layer, shape, width, and initialization: the first parameter step coincides with Muon+Moonlight's, the recipe learning rate and weight decay transfer with no additional sweep, and the ratio is width-invariant at initialization in the Moonlight+$\mu$P sense (\cref{cor:cal}). Afterwards it moves as $r_{\ell,t}\approx\|W_{\ell,t}\|_F/\|W_{\ell,0}\|_F$ in early training (\cref{lem:lars})---the LARS-style weight-norm signal, now referenced to each layer's own initialization.

Second, the adaptation should be bounded. A weight-norm numerator creates a potential feedback loop---growing weights raise $r_\ell$, amplifying the step, growing the weights further---and decoupled weight decay~\citep{loshchilov2019adamw} does not close it, because its pull-back $(1-\eta\lambda)W$ does not scale with $r_\ell$. Coupling the decay into the scaled update---putting $\lambda W_\ell$ \emph{inside} $D_{\ell,t}$, as \cref{def:principle} already requires---does: as $\|W_\ell\|_F$ grows the denominator becomes $\approx c_\ell\lambda\|W_\ell\|_F$ and the ratio approaches the finite ceiling $1/(c_\ell\lambda)$ (\cref{lem:ceiling}), while a conventional clip $\hat r_{\ell,t}=\clip(r_{\ell,t},r_{\min},r_{\max})$ guards the transient before the ceiling engages. \Cref{sec:req-ablation} deletes each design choice in turn and reports what breaks.

\subsection{The resulting algorithm}
\label{sec:algorithm}

\begin{algorithm}[t]
\caption{The OrScale update for matrix-valued parameters. The shape factor $s_\ell$ and the one-time calibration are configuration choices, not separate methods: language-model training uses Moonlight's $s_\ell$ with calibration on; the vision benchmark uses $s_\ell=1$ with calibration off.}
\label{alg:orscale}
\small
\begin{algorithmic}[1]
\Require LR $\eta_t$, momentum $\mu\in[0,1)$, weight decay $\lambda\ge 0$, NS iterations $k$ (default $5$), $\varepsilon=10^{-6}$; shape factor $s_\ell$ ($0.2\sqrt{\max(m_\ell,n_\ell)}$ for LM training, else $1$); calibration on/off (on for LM training); clip bounds $r_{\min},r_{\max}$ (defaults $0.1,5.0$ for LM training; $0.5,1.5$ otherwise).
\For{each matrix parameter $W_\ell\in\R^{m_\ell\times n_\ell}$}
  \State $G_{\ell,t}\gets \nabla_{W_\ell}\mathcal L_t$;\quad $M_{\ell,t}\gets \mu M_{\ell,t-1}+G_{\ell,t}$;\quad $\widetilde M_{\ell,t}\gets \mu M_{\ell,t}+G_{\ell,t}$
  \State $Q_{\ell,t}\gets \NS_k(\widetilde M_{\ell,t})$ \Comment{approximate polar factor}
  \If{calibration on \textbf{and} $c_{\ell}$ uninitialized}
    \State $c_{\ell}\gets \|W_{\ell,t}\|_F/(\|\lambda W_{\ell,t}+s_\ell Q_{\ell,t}\|_F+\varepsilon)$ \Comment{calibration anchor (\cref{def:calib})}
  \ElsIf{calibration off}
    \State $c_{\ell}\gets 1$
  \EndIf
  \State $D_{\ell,t}\gets \lambda W_{\ell,t}+s_\ell Q_{\ell,t}$ \Comment{applied direction, WD coupled}
  \State $r_{\ell,t}\gets \|W_{\ell,t}\|_F/(c_{\ell}\,\|D_{\ell,t}\|_F+\varepsilon)$ \Comment{fp32 cast around the division}
  \State $\hat r_{\ell,t}\gets\clip(r_{\ell,t},r_{\min},r_{\max})$;\qquad $W_{\ell,t+1}\gets W_{\ell,t}-\eta_t\,\hat r_{\ell,t}\,D_{\ell,t}$
\EndFor
\State Update non-matrix parameters with the AdamW fallback.
\end{algorithmic}
\end{algorithm}

The update, written out with the calibration $c_\ell$ of \cref{def:calib}, is
\begin{equation}
r_{\ell,t}=\frac{\|W_{\ell,t}\|_F}{c_\ell\,\|\lambda W_{\ell,t}+s_\ell Q_{\ell,t}\|_F+\varepsilon},\qquad
W_{\ell,t+1}=W_{\ell,t}-\eta_t\,\hat r_{\ell,t}\,(\lambda W_{\ell,t}+s_\ell Q_{\ell,t}).
\label{eq:orscale-lm}
\end{equation}
For language-model training we instantiate OrScale with Moonlight's shape factor $s_\ell=0.2\sqrt{\max(m_\ell,n_\ell)}$ and enable the calibration, following Moonlight's RMS-matching convention so that the AdamW-tuned recipe---learning rate and weight decay alike---carries over unchanged; the weight-decay coupling inside $D_{\ell,t}$ is part of the core update in both configurations, not an LM-specific choice. On the CIFAR-10 benchmark, where the standard protocol sweeps every optimizer family over its own learning-rate grid, there is no tuned recipe to preserve, so we run \cref{eq:orscale-lm} with $s_\ell=1$ and $c_\ell=1$: the LARS template with the substitution $G\to Q$.

\textbf{Cost, state, and degenerate cases.}
With calibration on, OrScale adds one fp32 scalar per matrix layer (none otherwise) and two Frobenius norms per step, both negligible next to Newton--Schulz; measured wall-clock at 16B is at parity with Muon+Moonlight (\cref{tab:moe}). \Cref{alg:orscale} contains the baselines as limits: $r_{\min}=r_{\max}=1$ with $s_\ell=1$ is Muon; $r_{\min}=r_{\max}=1$ with the Moonlight $s_\ell$ is Muon+Moonlight; $\lambda=0$ removes weight decay while keeping the scaling.

\section{Convergence theory}
\label{sec:theory}

Full proofs are in \cref{app:proofs}. We minimize $f(W)=\E_s[\ell(W,s)]$ over $h$ matrix-valued blocks $W_\ell\in\R^{m_\ell\times n_\ell}$; write $p_\ell=\min(m_\ell,n_\ell)$ and $P=\sum_\ell p_\ell$. Assumptions: \textbf{(A1)}~layerwise smoothness with constants $L_\ell$, $\bar L_p{=}(\sum_\ell L_\ell p_\ell)/P$; \textbf{(A2)}~bounded variance, $\E\|G_{\ell,t}-\nabla_\ell f\|_F^2\le\sigma_\ell^2/b$; \textbf{(A3)}~bounded nuclear noise, $\E\|\xi_{\ell,t}\|_*\le\sqrt{p_\ell}\,\sigma_\ell/\sqrt b$ (implied by~(A2)); \textbf{(A4)}~the headline statements set $\mu=0$, and \cref{thm:basic} additionally $\lambda=0$ ($\lambda>0$ restored in \cref{prop:wd}; the $\mu>0$ extension follows LAMB~\citep{you2020lamb}). Convergence criterion: $\Psi(W)=\tfrac{1}{\sqrt P}\sum_\ell\|\nabla_\ell f(W)\|_*$.

Two properties of polar factors underpin the analysis (\cref{lem:polar}): the inner-product identity $\langle A,Q_A\rangle=\|A\|_*$ with its dual bound $|\langle B,Q_A\rangle|\le\|B\|_*$, and the norm identity $\|Q_A\|_F=\sqrt{\rank A}$. Together these yield an expected-descent bound: $\E\langle\nabla_\ell f,Q_{\ell,t}\rangle\ge\|\nabla_\ell f\|_*-\sqrt{p_\ell}\,\sigma_\ell/\sqrt b$ (\cref{lem:descent}).

\begin{theorem}[Rate for any clipped multiplier]
\label{thm:basic}
Under (A1)--(A4), for any update $W_{\ell,t+1}=W_{\ell,t}-\eta\hat r_{\ell,t}Q_{\ell,t}$ with $\hat r_{\ell,t}\in[r_{\min},r_{\max}]$ arbitrary (possibly sample-dependent), the stepsize $\eta=\sqrt{2(f_0-f^*)/(r_{\max}^2T\sum_\ell L_\ell p_\ell)}$ and $W_a$ uniform on $\{W_1,\dots,W_T\}$ give
\[
\big(\E\,\Psi(W_a)\big)^2 \;\le\; \frac{4r_{\max}^2\,(f_0-f^*)\,\bar L_p}{r_{\min}^2\,T}\;+\;\frac{2r_{\max}^2}{r_{\min}^2}\cdot\frac{\big(\sum_\ell\sqrt{p_\ell}\,\sigma_\ell\big)^2}{P\,b}.
\]
\end{theorem}

\Cref{thm:basic} treats $\hat r$ as a black box: it guarantees convergence for any clipped multiplier but cannot distinguish OrScale from Muon. The comparison requires a finer quantity. Let $a_\ell=\|\nabla_\ell f\|_*$ and $b_\ell=L_\ell p_\ell$, and define the \emph{descent functional} $\Phi(\hat r)=\big(\sum_\ell\hat r_\ell a_\ell\big)^2/\sum_\ell b_\ell\hat r_\ell^2$ for per-layer multipliers $\hat r$; Muon corresponds to $\hat r\equiv 1$. The \emph{effective gain} is $\kappa_{\mathrm{eff}}(\hat r)=\Phi(\hat r)/\Phi(\mathbf 1)$.

\begin{theorem}[Layer-adaptive headroom]
\label{thm:opt}
$\max_{\hat r}\Phi(\hat r)=\sum_\ell a_\ell^2/b_\ell$ at $\hat r_\ell^*\propto a_\ell/b_\ell$, and $\kappa_{\mathrm{layer}}:=\Phi(\hat r^*)/\Phi(\mathbf 1)\ge 1$ with equality iff $a_\ell/b_\ell$ is constant across layers. Geometrically, with $x_\ell=a_\ell/\sqrt{b_\ell}$ and $y_\ell=\hat r_\ell\sqrt{b_\ell}$: $\kappa_{\mathrm{eff}}(\hat r)=\cos^2\theta_{xy}/\cos^2\theta_{x,\sqrt b}$ (\cref{lem:geom})---an optimizer beats Muon iff its weighted multipliers align better with the optimum than uniform scaling does.
\end{theorem}

A two-layer example makes the scaling concrete: with $L_2=\beta L_1$ and $a_2=\alpha a_1$, $\kappa_{\mathrm{layer}}=(1+\beta)(1+\alpha^2/\beta)/(1+\alpha)^2$ evaluates to $1.67$, $3.03$, $9.10$ at $(\alpha,\beta)\in\{(0.1,1),(1,10),(0.1,10)\}$---\emph{the headroom grows with cross-layer heterogeneity}. Whether OrScale captures any of this headroom depends on two statistics of its logged multipliers: \textbf{(S1)}~heterogeneity, $\nu_{\hat r}=\Var_\ell[\hat r_\ell\sqrt{b_\ell}]/\E_\ell[\hat r_\ell\sqrt{b_\ell}]^2>0$; and \textbf{(S2)}~alignment, $\rho=\corr_\ell(\hat r_\ell\sqrt{b_\ell},\,x_\ell)>0$.

\begin{theorem}[OrScale-specific gain]
\label{thm:gain}
Under (A1)--(A4), (S1), (S2),
$\;\kappa_{\mathrm{eff}}^{\mathrm{OrScale}} \ge 1+\rho^2\nu_{\hat r}\,\dfrac{1+\nu_{\sqrt b}}{1+\nu_{\sqrt b}/\rho^2}-\delta(\lambda)$,
where $\nu_{\sqrt b}=\Var_\ell[\sqrt{b_\ell}]/\E_\ell[\sqrt{b_\ell}]^2$ and $\delta(\lambda)=O(\lambda^2\|W\|_F^2/\|Q\|_F^2)$.
\end{theorem}

\begin{theorem}[Collapse of unit-inconsistent ratios]
\label{thm:collapse}
If the clipped ratio saturates ($\hat r_{\ell,t}=r_{\max}$ for all $\ell$ at almost every step---the logged behavior of the raw-momentum denominator variants, \cref{sec:req-ablation}), then $\Phi(r_{\max}\mathbf 1)=\Phi(\mathbf 1)$ and $\kappa_{\mathrm{eff}}=1$: a uniformly saturated multiplier is Muon with a relabeled learning rate.
\end{theorem}

\textbf{A falsifiable prediction.}
\Cref{thm:opt,thm:gain} reduce the gain over Muon to two measurable statistics: heterogeneity and alignment of the logged trust ratios. On CIFAR-10/DavidNet diagnostics, $\nu_{\hat r}\approx 0.3$, $\rho\approx 0.5$, $\nu_{\sqrt b}\approx 0.5$ give $\kappa_{\mathrm{eff}}\in[1.06,1.15]$, consistent with the observed $+0.35$-point gain (\cref{sec:cifar}). The same logic predicts \emph{where} the gain should grow: architectures whose layer classes see maximally different update statistics. Mixture-of-experts transformers are the sharpest available case---routers, sparsely-activated experts, and dense attention blocks differ in activation frequency, gradient scale, and curvature---so (S1)--(S2) should strengthen, and with them the gap over Muon+Moonlight; the prediction was recorded before the 16B-A3B run was launched, and \cref{sec:moe-heterogeneity} reports the test. \Cref{cor:cal} collects the calibrated update's remaining properties (anchor exactness, width-invariance, weight-norm ceiling, early-training LARS signal).

\section{Experiments}
\label{sec:experiments}

\subsection{Setup and the zero-sweep protocol}
\label{sec:protocol}

\textbf{Models and data.} Language models are trained on FineWeb-Edu~\citep{penedo2024fineweb} in two regimes: \emph{dense}, GPT-2-style pre-norm transformers~\citep{radford2019gpt2} (RMSNorm, RoPE, ReLU$^2$ MLP, tied embeddings, vocabulary 50{,}304) at 125M/399M/545M/1.1B following the Moonlight scaling recipe~\citep{liu2025moonlight}; and \emph{MoE}, the Moonlight-16B-A3B architecture (DeepSeek-V3-style~\citep{deepseekai2024deepseekv3}: 27 layers, multi-head latent attention, 64 routed experts with top-6 routing plus 2 shared experts, aux-loss-free load balancing; 15.96B total, 2.24B active parameters) on 10B tokens at sequence length 8192. Vision cross-check: CIFAR-10 on the 24-epoch DavidNet speedrun ($\approx$6.5M parameters), three seeds per cell. Complete per-scale settings are in \cref{app:hyperparams}, compute in \cref{app:compute}.

\textbf{The zero-sweep protocol.} Our evaluation deliberately stresses the calibration anchor. At 399M, 545M, 1.1B, and 16B-A3B, \emph{every} optimizer runs at the Moonlight-recipe learning rate for that scale ($9.503{\times}10^{-4}$, $9.143{\times}10^{-4}$, $8.561{\times}10^{-4}$, $4.2{\times}10^{-4}$), with weight decay $0.1$ and momentum $0.95$: the recipe is tuned for AdamW, shared by Muon+Moonlight via RMS matching, and inherited by OrScale via the calibration anchor. At 125M the baselines were additionally tuned over a four-point LR grid and report their best cell. \emph{OrScale was never swept at any language-model scale in this paper}; its clip bounds stay at the defaults $(0.1,5.0)$ throughout. Any OrScale win therefore comes with zero tuning advantage---if anything the protocol favors the baselines.

\subsection{Dense scaling: 125M--1.1B}
\label{sec:dense}

\begin{table}[t]
\centering
\caption{FineWeb-Edu dense pre-training: final validation cross-entropy (lower is better; \textbf{bold} best per row). All optimizers share the Moonlight-recipe LR shown per row (125M baselines: best of a four-point grid; OrScale everywhere: recipe LR, unswept). Cells with $\pm\sigma$ are means over three seeds; others are single-seed.}
\label{tab:dense}
\small
\setlength{\tabcolsep}{4.5pt}
\begin{tabular}{l l c c c c}
\toprule
Scale & Tokens & Recipe LR & AdamW & Muon+Moonlight & \textbf{OrScale} (ours) \\
\midrule
125M & 5.24B & $1.0\times10^{-3}$ & $3.3721$ & $3.2319$ & $\mathbf{3.2120\pm0.0017}$ \\
399M & 8.92B & $9.503\times10^{-4}$ & $2.9966$ & $\mathbf{2.9198\pm0.0016}$ & $2.9223\pm0.0022$ \\
545M & 14.04B & $9.143\times10^{-4}$ & $2.9235$ & $2.8141\pm0.0013$ & $\mathbf{2.8052\pm0.0013}$ \\
1.1B & 28.54B & $8.561\times10^{-4}$ & $2.7304$ & $2.6360$ & $\mathbf{2.6251}$ \\
\bottomrule
\end{tabular}
\end{table}

\begin{figure}[t]
\centering
\includegraphics[width=\linewidth]{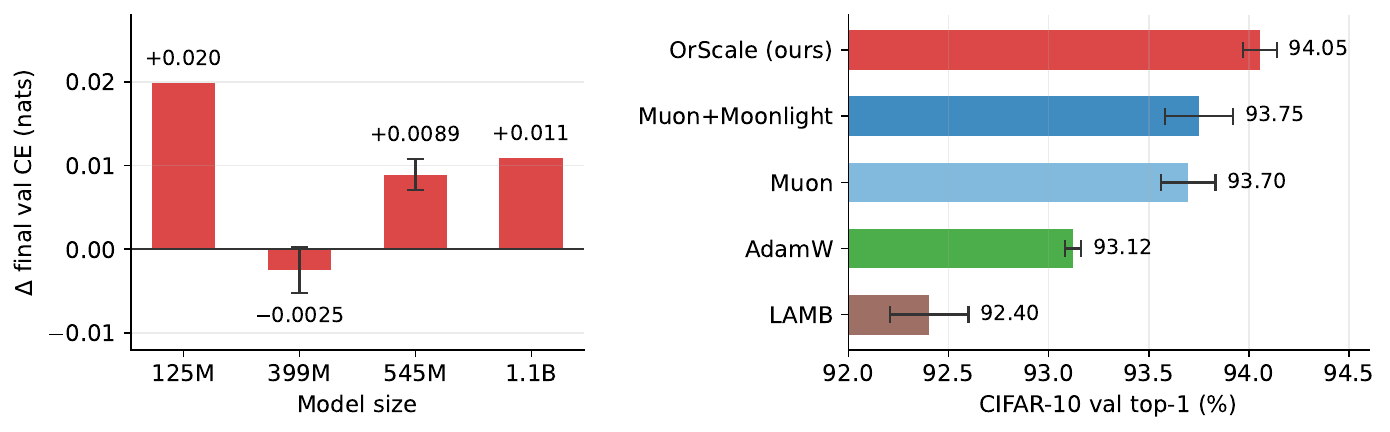}
\caption{\emph{Left:} per-scale gap (Muon+Moonlight $-$ OrScale) on dense FineWeb-Edu; positive means OrScale better. Values match \cref{tab:dense}: three-seed means at 399M/545M and for OrScale at 125M (error bars where both sides have seed replicates: per-side stds combined in quadrature); single-seed elsewhere; the 399M cell is within noise and reported as a tie. \emph{Right:} CIFAR-10/DavidNet val top-1 at each optimizer family's best learning rate (three seeds, $\pm1\sigma$, per-family grids); OrScale runs the vision configuration ($s_\ell{=}1$, calibration off; \cref{sec:cifar}, full sweep in \cref{app:cifar}).}
\label{fig:dense-cifar}
\end{figure}

\Cref{tab:dense} spans a $48\times$ compute range ($0.046\to2.18$ PFLOP-days). Against Muon+Moonlight, OrScale wins at 125M ($+0.020$), 545M ($+0.0089\pm0.0018$ over three seeds---significant), and 1.1B ($+0.011$), and ties at 399M ($-0.0025\pm0.0027$---within noise, reported as a tie); \cref{fig:dense-cifar} (left) plots the gaps at the seed-matched protocol. Against AdamW it wins everywhere by $+0.074$ to $+0.160$ nats. Fitted log--log exponents ($-0.054$/$-0.053$/$-0.052$; \cref{app:lr-sweep}) show the advantage on \emph{dense} models is scale-stable rather than growing---the theory attributes the gain to heterogeneity, and a homogeneous stack of identical blocks offers only a fixed amount of it; the prediction is tested where heterogeneity is large, next. A matched 125M LR sweep (shared grid $\{3{\times}10^{-4},10^{-3},3{\times}10^{-3},10^{-2}\}$) shows OrScale at or below Muon+Moonlight at every stable cell, with Muon+Moonlight diverging at the top cell OrScale survives (\cref{app:lr-sweep}).

\subsection{16B-A3B mixture-of-experts: the heterogeneity stress test}
\label{sec:moe}

\begin{table}[t]
\centering
\caption{Moonlight-16B-A3B MoE, 10B FineWeb-Edu tokens, 8$\times$H20, bf16, ZeRO-1, seed 42. All arms use the Moonlight-recipe settings for this scale (LR $4.2{\times}10^{-4}$, wd $0.1$, $\mu{=}0.95$, warmup 60, cosine), shared by AdamW and Muon+Moonlight by construction and inherited unchanged by OrScale.}
\label{tab:moe}
\small
\begin{tabular}{l c c c c}
\toprule
Optimizer & $r_{\min}/r_{\max}$ & Final val CE & $\Delta$ vs.\ Muon+ML & Wall-clock \\
\midrule
AdamW & --- & $5.2002$ & $-1.7287$ & $3.93$ d \\
Muon + Moonlight & --- & $3.4715$ & --- & $4.98$ d \\
OrScale ($r_{\max}{=}10$) & $0.1/10$ & $3.3556$ & $+0.1159$ & $4.92$ d \\
\textbf{OrScale} & $0.1/5$ & $\mathbf{3.3410}$ & $+0.1305$ & $4.92$ d \\
\bottomrule
\end{tabular}
\end{table}

\Cref{fig:moe-headline} and \cref{tab:moe} report the headline test. At $14.5\times$ the parameter count of our largest dense run, with every hyperparameter inherited verbatim, OrScale improves on Muon+Moonlight by $0.130$ nats ($3.8\%$ relative)---an order of magnitude beyond the dense-regime gaps. OrScale leads at every validation checkpoint from 2.5B tokens onward and passes Muon+Moonlight's \emph{final} loss with a third of the token budget to spare; wall-clock is at parity ($425{,}018$\,s vs.\ $430{,}328$\,s), replacing the usual ${<}1\%$-overhead assertion with a measurement. The $r_{\max}{=}10$ arm doubles the clip ceiling and still beats Muon+Moonlight while losing to $r_{\max}{=}5$: the win is not an artifact of a tight clip, and the coupled-WD ceiling is doing real work. AdamW at the shared recipe trails by $1.73$ nats, consistent with the Moonlight paper's own AdamW-vs-Muon comparison at this architecture; we report it as the recipe-for-recipe reference point (no arm received a per-optimizer retune).

\textbf{The theory's conditions, measured.}\label{sec:moe-heterogeneity}
The gain condition of \cref{thm:gain} is checkable from logs, and this run checks it: the logged trust ratios separate cleanly and persistently by layer class---routed-expert projections at the ceiling within 1.5B tokens, router gates near $2.2$--$2.3$ and never clipping, attention in between---exactly the heterogeneity (S1) requires, aligned with the layer classes whose update statistics differ most, as (S2) requires. The late-training clip fraction ($33\%$ of layers at a bound) is the coupled-WD ceiling engaging for the layer class that wants the largest relative steps, not the uniform saturation pathology of \cref{thm:collapse}. Per-layer trajectories, calibration-constant statistics, and the weaker dense-125M counterparts of these diagnostics are in \cref{app:moe-extra} (\cref{fig:moe-trust}).

\subsection{Calibration is load-bearing}
\label{sec:calibration-ablation}

\begin{table}[t]
\centering
\caption{Calibration ablation (single-seed log values; three-seed means in \cref{tab:dense}). The uncalibrated variant runs \cref{def:principle} with the Moonlight shape factor but no anchor ($c_\ell{=}1$) and must retune its matrix LR upward to train well; calibrated OrScale runs at the untouched recipe LR everywhere.}
\label{tab:calibration}
\small
\begin{tabular}{l l l c}
\toprule
Scale & Variant & Matrix LR & Final val CE \\
\midrule
399M & Muon+Moonlight (reference) & recipe ($9.503{\times}10^{-4}$) & $2.9183$ \\
399M & OrScale, \emph{uncalibrated} & $2.5\times$ recipe ($2.376{\times}10^{-3}$) & $2.9939$ \\
399M & \textbf{OrScale} (calibrated) & recipe, unswept & $\mathbf{2.9247}$ \\
\midrule
125M & OrScale, \emph{uncalibrated} & $3\times$ base ($3{\times}10^{-3}$, tuned) & $3.2232$ \\
125M & \textbf{OrScale} (calibrated) & base ($10^{-3}$), unswept & $\mathbf{3.2120}$ \\
125M & OrScale, calibration in bf16 & base ($10^{-3}$) & $3.2158$ \\
\bottomrule
\end{tabular}
\end{table}

If the calibration constant were cosmetic, removing it would cost a learning-rate sweep and nothing else. \Cref{tab:calibration} shows it costs more: the uncalibrated ratio needs a $2.5$--$3\times$ matrix-LR retune to train well (its uncorrected $r_{\ell,0}\propto 1/\sqrt{\mathrm{fan\mbox{-}in}}$ shrinks early steps several-fold) and \emph{still} finishes $0.069$ nats behind the calibrated variant at 399M---behind even the Muon+Moonlight baseline. Calibration converts the trust ratio from a knob that must be re-tuned per model into a drop-in that inherits the recipe, and it is also worth loss. Computing the anchor in bf16 rather than fp32 costs a small but systematic $0.004$ nats at 125M; hence the fp32 cast in \cref{alg:orscale}.

\subsection{Design-choice ablation}
\label{sec:req-ablation}

\begin{table}[t]
\centering
\caption{Deleting one design choice at a time; each deletion produces a distinct, logged symptom (accuracy alone can mislead: the top row ties OrScale on CIFAR while its ratio carries no update-scale information). CIFAR: val top-1 at the best LR, three seeds. FW-125M: final CE at each variant's best LR under the instrumented 125M ablation sweep that produced the logged symptoms, single seed; that setup differs from the strict-recipe runs (OrScale there: $3.2120$, \cref{tab:dense}), so FW values are comparable within this table only. Symptoms from 125M diagnostics (\cref{app:req-details}).}
\label{tab:req-ablation}
\footnotesize
\setlength{\tabcolsep}{2.5pt}
\begin{tabular}{l l l c c}
\toprule
Variant & Omitted & Logged symptom & CIFAR-10 (\%) & FW-125M (CE) \\
\midrule
$\|Q_\ell\|_F$ denominator & denominator & ratio $=$ shape constant (\cref{lem:polar}) & $94.06\pm0.09$ & $3.231$ \\
$\|\widetilde M_\ell\|_F$ denominator & denominator & $\hat r{=}r_{\max}$ on $99.75$--$100\%$ of steps & $93.84\pm0.15$ & $3.252$ \\
$\RMS(\widetilde M_\ell)$ + Moonlight & denominator & $\hat r{=}r_{\max}$ on $99.5$--$100\%$ of steps & $93.83\pm0.17$ & $3.230$ \\
Uncalibrated ratio & calibration & needs $2.5$--$3\times$ LR retune (\cref{tab:calibration}) & --- & $3.2232$ \\
Decoupled weight decay & WD coupling & $10$--$50\times$ weight-norm growth & $93.55\pm0.12$ & $3.305$ \\
\midrule
\textbf{OrScale} ($s_\ell{=}1$) & --- & in-band, layer-structured & $94.05\pm0.08$ & $3.228$ \\
\textbf{OrScale} (LM config.) & --- & in-band, layer-structured & $94.05\pm0.12$ & $\mathbf{3.226}$ \\
\bottomrule
\end{tabular}
\end{table}

\Cref{tab:req-ablation} deletes one design choice at a time. The deleted variants coincide with hybrid designs that have been suggested for combining trust ratios with Muon, or that we implemented first ourselves; we report them once, here, as ablations of \cref{sec:derivation}'s design choices rather than as motivation. Dropping the applied-direction denominator either removes adaptivity ($\|Q\|_F$ is a shape constant) or mis-scales the ratio by orders of magnitude so it pins at the clip bound on effectively every step---the empirical premise of \cref{thm:collapse}, under which $\kappa_{\mathrm{eff}}=1$ and the ``adaptive'' method is Muon with a relabeled LR. Dropping the calibration severs learning-rate transfer (\cref{tab:calibration}). Decoupling weight decay triggers the runaway feedback loop, with $10$--$50\times$ weight-norm growth on \texttt{mlp} and output projections. Removing the clip entirely is survivable---transient ratios reach $9.6$ before the coupled-WD ceiling engages, ending $0.007$ nats behind---so the safety-critical boundedness lives in the coupled-WD ceiling, not the clip; at 16B, doubling $r_{\max}$ moves the loss by only $0.015$ nats (\cref{tab:moe}).

\subsection{Vision cross-check: CIFAR-10 / DavidNet}
\label{sec:cifar}

On the 24-epoch DavidNet speedrun (three seeds, per-family LR grids), OrScale ranks first at $94.05\pm0.08\%$ val top-1 (\cref{fig:dense-cifar}, right), improving Muon ($93.70\pm0.14$) by $+0.35$ and Muon+Moonlight ($93.75\pm0.17$) by $+0.30$ points ($\approx 5\%$ residual-error reduction); AdamW and LAMB trail at $93.12\pm0.04$ and $92.40\pm0.20$. The result cross-validates the derivation outside the LLM regime and supplies the diagnostics behind the numerical $\kappa_{\mathrm{eff}}$ estimate of \cref{sec:theory}; the full sweep, LR-sensitivity analysis, and per-cell table are in \cref{app:cifar}.

\section{Discussion and conclusion}
\label{sec:discussion}

Muon fixes each update's direction; the remaining per-layer degree of freedom is one scalar per step. We derived that scalar by adapting the LARS/LAMB trust-ratio template to the orthogonalized setting, validating each design choice by a deletion that fails in its own way; the calibration anchor lets a Moonlight-tuned recipe transfer with zero extra sweeps. The theory conditions the gain over Muon on logged statistics rather than oracle constants, predicting growth with architectural heterogeneity---and the sharpest available test agrees: $+0.130$ nats on a 16B-A3B MoE at recipe hyperparameters and parity wall-clock, with the logged ratios exhibiting the layer-class structure the theory requires.

\textbf{Limitations.} The 16B-A3B comparison is single-seed (dense three-seed replication brackets seed noise at $\pm0.002$--$0.003$ nats, ${\sim}50\times$ below the MoE gap); results are FineWeb-Edu validation cross-entropy without downstream evaluations; claims concern autoregressive next-token pre-training with the Moonlight recipe family, and the dense-regime gains are small ($0.009$--$0.020$ nats where significant); what the theory does and does not claim is delimited in \cref{app:proofs}.

\subsubsection*{Ethics statement}
This work concerns optimizer design for neural-network pre-training. It involves no human subjects and no new data collection: all experiments use the publicly released FineWeb-Edu corpus and the CIFAR-10 benchmark under their standard terms. The primary ethical consideration is the dual-use nature of training efficiency---OrScale reduces compute, cost, and energy per unit of pre-training progress, and thereby also lowers the barrier to training models for harmful uses; \cref{app:impact} expands on this, and \cref{app:compute} reports the project's full compute footprint. The method introduces no new dataset, surveillance, or model-release pathway beyond the optimizer implementation.

\subsubsection*{Reproducibility statement}
The full training and analysis code, one configuration file per reported result row, and the raw training logs for the 16B-A3B runs are open-sourced at \url{https://github.com/NUS-HPC-AI-Lab/OrScale}; every table and figure in the paper regenerates from logged data via the included scripts. \Cref{app:hyperparams} specifies architectures, schedules, batch geometry, and every optimizer hyperparameter per scale (none are deferred to external sources); \cref{app:compute} reports hardware and wall-clock; proofs of all statements are in \cref{app:proofs}.

\subsubsection*{AI use statement}
Large language models were used as writing and engineering assistants in preparing this manuscript (restructuring and editing text, LaTeX, and log-analysis/plotting scripts). The method, experiments, and theory originate with the authors, who verified all technical claims, numbers, and proofs.

\bibliography{refs}
\bibliographystyle{iclr2027_conference}

\appendix

\section{Proofs}
\label{app:proofs}

\subsection{Polar-factor algebra and expected descent}

\begin{lemma}[Polar-factor algebra]
\label{lem:polar}
For $A=U\Sigma V^\top$ of rank $k\le\min(m,n)$ and $Q_A=\Ortho(A)=UV^\top$: (i) $\langle A,Q_A\rangle=\|A\|_*$; (ii) $|\langle B,Q_A\rangle|\le\|B\|_*$ for any $B$; (iii) $\|Q_A\|_F=\sqrt k$ and $\|Q_A\|_{\mathrm{op}}=1$.
\end{lemma}

\begin{proof}
(i) $\langle A,Q_A\rangle=\mathrm{tr}(V\Sigma U^\top UV^\top)=\mathrm{tr}(\Sigma)=\|A\|_*$. (ii) Von Neumann's trace inequality gives $|\mathrm{tr}(B^\top Q_A)|\le\sum_i\sigma_i(B)\,\sigma_i(Q_A)\le\sum_{i\le k}\sigma_i(B)\le\|B\|_*$ since all singular values of $Q_A$ are $1$ (first $k$) or $0$. (iii) immediate from the singular values of $Q_A$.
\end{proof}

\begin{lemma}[Expected descent]
\label{lem:descent}
Let $G_\ell=\nabla_\ell f+\xi_\ell$ with $\E\|\xi_\ell\|_F^2\le\sigma_\ell^2/b$ and $Q_\ell=\Ortho(G_\ell)$. Then
$\E\langle\nabla_\ell f,Q_\ell\rangle\ge\|\nabla_\ell f\|_*-\sqrt{p_\ell}\,\sigma_\ell/\sqrt b$.
\end{lemma}

\begin{proof}
$\langle\nabla_\ell f,Q_\ell\rangle=\langle G_\ell,Q_\ell\rangle-\langle\xi_\ell,Q_\ell\rangle=\|G_\ell\|_*-\langle\xi_\ell,Q_\ell\rangle$ by \cref{lem:polar}(i). By Jensen on the convex norm, $\E\|G_\ell\|_*\ge\|\E G_\ell\|_*=\|\nabla_\ell f\|_*$. By \cref{lem:polar}(ii), $|\langle\xi_\ell,Q_\ell\rangle|\le\|\xi_\ell\|_*\le\sqrt{p_\ell}\|\xi_\ell\|_F$, so $\E|\langle\xi_\ell,Q_\ell\rangle|\le\sqrt{p_\ell}\sqrt{\E\|\xi_\ell\|_F^2}\le\sqrt{p_\ell}\sigma_\ell/\sqrt b$ by Cauchy--Schwarz.
\end{proof}

\subsection{Proof of \texorpdfstring{\cref{thm:basic}}{Theorem 1}}
\label{app:thm-basic}

\textbf{Step 1 (smoothness).} Layerwise smoothness (A1) with $W_{\ell,t+1}-W_{\ell,t}=-\eta\hat r_{\ell,t}Q_{\ell,t}$ gives
\[
f(W_{t+1})\le f(W_t)-\eta\sum_\ell\hat r_{\ell,t}\,\langle\nabla_\ell f(W_t),Q_{\ell,t}\rangle+\frac{\eta^2}{2}\sum_\ell L_\ell\,\hat r_{\ell,t}^2\,\|Q_{\ell,t}\|_F^2.
\]

\textbf{Step 2 (second order).} $\|Q_{\ell,t}\|_F^2\le p_\ell$ (\cref{lem:polar}(iii)) and $\hat r\le r_{\max}$ bound the last term by $\tfrac{\eta^2r_{\max}^2}{2}\sum_\ell L_\ell p_\ell$.

\textbf{Step 3 (first order, no sign assumption).} Write $x_{\ell,t}=\langle\nabla_\ell f(W_t),Q_{\ell,t}\rangle$ and let $x^-=\max(0,-x)$ denote the negative part. The sign of $x_{\ell,t}$ is \emph{not} controlled---when noise dominates, $Q_{\ell,t}$ can be an ascent direction---and $\hat r_{\ell,t}$ may depend on the same sample, so $\hat r_{\ell,t}x_{\ell,t}\ge r_{\min}x_{\ell,t}$ is false in general. Instead we use the pointwise inequality, valid for any $\hat r\in[r_{\min},r_{\max}]$:
\[
\hat r_{\ell,t}\,x_{\ell,t}\;\ge\;r_{\min}\,x_{\ell,t}\;-\;(r_{\max}-r_{\min})\,x_{\ell,t}^-,
\]
verified by inspection in the two sign cases ($x\ge0$: $\hat rx\ge r_{\min}x$; $x<0$: $\hat rx\ge r_{\max}x=r_{\min}x-(r_{\max}-r_{\min})x^-$).

\textbf{Step 4 (the negative part is pure noise).} By \cref{lem:polar}(i)--(ii),
\[
x_{\ell,t}=\|G_{\ell,t}\|_*-\langle\xi_{\ell,t},Q_{\ell,t}\rangle\;\ge\;-\|\xi_{\ell,t}\|_*,
\qquad\text{hence}\qquad
x_{\ell,t}^-\le\|\xi_{\ell,t}\|_*,
\]
and by (A3), $\E[x_{\ell,t}^-]\le\sqrt{p_\ell}\,\sigma_\ell/\sqrt b$. Combining Steps 3--4 with \cref{lem:descent},
\[
\E\big[\hat r_{\ell,t}x_{\ell,t}\big]\;\ge\;r_{\min}\Big(\|\nabla_\ell f\|_*-\frac{\sqrt{p_\ell}\sigma_\ell}{\sqrt b}\Big)-(r_{\max}-r_{\min})\frac{\sqrt{p_\ell}\sigma_\ell}{\sqrt b}
\;=\;r_{\min}\|\nabla_\ell f\|_*-r_{\max}\frac{\sqrt{p_\ell}\sigma_\ell}{\sqrt b}.
\]

\textbf{Step 5 (telescope).} Substituting into Step 1, taking total expectation, summing over $t=1,\dots,T$, and rearranging,
\[
\frac{1}{T}\sum_{t=1}^T\sum_\ell\E\|\nabla_\ell f(W_t)\|_*
\;\le\;
\frac{f_0-f^*}{\eta\,r_{\min}\,T}
+\frac{r_{\max}}{r_{\min}}\cdot\frac{\sum_\ell\sqrt{p_\ell}\,\sigma_\ell}{\sqrt b}
+\frac{\eta\,r_{\max}^2}{2r_{\min}}\sum_\ell L_\ell p_\ell.
\]

\textbf{Step 6 (stepsize).} With $\eta=\sqrt{2(f_0-f^*)/(r_{\max}^2T\sum_\ell L_\ell p_\ell)}$ the first and third terms balance at $\tfrac{r_{\max}}{r_{\min}}\sqrt{2(f_0-f^*)\sum_\ell L_\ell p_\ell/T}$. Dividing by $\sqrt P$ to convert to the $\Psi$ criterion, squaring, and using $(u+v)^2\le 2u^2+2v^2$ yields the claim. \qed

\begin{remark}[Sign-free first-order bound]
\label{rem:fix}
Because $\hat r_{\ell,t}$ may depend on the current sample and the sign of $\langle\nabla_\ell f,Q_{\ell,t}\rangle$ is not controlled, the first-order term is bounded via the pointwise inequality $\hat r x\ge r_{\min}x-(r_{\max}-r_{\min})\,x^-$ with $x^-\le\|\xi_{\ell,t}\|_*$ (Step~3 above). This inflates only the noise term by the factor $r_{\max}/r_{\min}$ visible in \cref{thm:basic}, leaving the $O(1/\sqrt T)$ rate unchanged. The factor is the worst case over arbitrary in-band multipliers---the price of a single statement covering both well-behaved and saturation-prone ratio designs (\cref{sec:req-ablation}). No sign assumption on $\langle\nabla_\ell f,Q_{\ell,t}\rangle$ is needed, and downstream statements (\cref{thm:opt,thm:gain,thm:collapse,cor:cal}) are unaffected, as none consume \cref{thm:basic}'s constant.
\end{remark}

\subsection{Coupled weight decay}

\begin{proposition}[Coupled-WD extension]
\label{prop:wd}
Under (A1)--(A3) and the update $W_{\ell,t+1}=W_{\ell,t}-\eta\hat r_{\ell,t}(\lambda W_{\ell,t}+s_\ell Q_{\ell,t})$ with $\lambda>0$, the descent inequality of \cref{app:thm-basic} acquires a first-order term $-\eta\hat r_{\ell,t}\lambda\langle\nabla_\ell f,W_{\ell,t}\rangle$ and second-order contributions bounded via $\lambda^2\|W_{\ell,t}\|_F^2+2\lambda s_\ell\langle W_{\ell,t},Q_{\ell,t}\rangle$. Along OrScale trajectories $\|W_{\ell,t}\|_F$ is bounded by the ceiling mechanism of \cref{lem:ceiling}, so these are $O(\lambda)$ corrections to the constants and the $O(1/\sqrt T)$ rate is preserved.
\end{proposition}

\begin{proof}[Proof sketch]
Substitute the coupled update into layerwise smoothness and split the cross terms; the bookkeeping follows the LAMB weight-decay extension~\citep[App.~B]{you2020lamb} with the nuclear-norm criterion in place of the $\ell_1$ criterion, using \cref{lem:polar}(ii) to bound $\langle W,Q\rangle$ by $\|W\|_*$.
\end{proof}

\subsection{Layer-adaptive headroom and the geometric identity}

\begin{proof}[Proof of \cref{thm:opt}]
$\Phi(\hat r)=\langle x,y\rangle^2/\|y\|^2$ with $x_\ell=a_\ell/\sqrt{b_\ell}$, $y_\ell=\hat r_\ell\sqrt{b_\ell}$; by Cauchy--Schwarz $\Phi(\hat r)\le\|x\|^2=\sum_\ell a_\ell^2/b_\ell$ with equality iff $y\parallel x$, i.e.\ $\hat r_\ell\propto a_\ell/b_\ell$. Muon's value is $\Phi(\mathbf 1)=(\sum_\ell a_\ell)^2/\sum_\ell b_\ell$, and $\kappa_{\mathrm{layer}}=\|x\|^2\|\sqrt b\|^2/\langle x,\sqrt b\rangle^2\ge1$ with equality iff $x\parallel\sqrt b$, i.e.\ $a_\ell/b_\ell$ constant.
\end{proof}

\begin{lemma}[Geometric form]
\label{lem:geom}
$\kappa_{\mathrm{eff}}(\hat r)=\cos^2\theta_{xy}/\cos^2\theta_{x,\sqrt b}$, directly from $\Phi(\hat r)=\|x\|^2\cos^2\theta_{xy}$.
\end{lemma}

\subsection{Proof of \texorpdfstring{\cref{thm:gain}}{the gain theorem}}
\label{app:thm-gain}

By \cref{lem:geom} it suffices to lower-bound $\cos^2\theta_{xy}/\cos^2\theta_{x,\sqrt b}$. Decompose
\[
y=\alpha\sqrt b+\beta\big(x-\pi_{\sqrt b}(x)\big)+\rho_\perp,
\]
where $\pi_{\sqrt b}$ is the orthogonal projection onto $\mathrm{span}(\sqrt b)$ and $\rho_\perp\perp\mathrm{span}(\sqrt b,x)$. The correlation $\rho=\corr_\ell(y_\ell,x_\ell)$ fixes $\beta=\rho\,\|y-\pi_{\sqrt b}(y)\|/\|x-\pi_{\sqrt b}(x)\|\cdot(1+o(1))$ after centering against $\sqrt b$; (S2) gives $\beta>0$. (S1) lower-bounds the off-$\sqrt b$ mass of $y$: $\|y-\pi_{\sqrt b}(y)\|^2/\|\pi_{\sqrt b}(y)\|^2\ge\nu_{\hat r}/(1+\nu_{\sqrt b})$ by the variance decomposition of $y_\ell=\hat r_\ell\sqrt{b_\ell}$ against the direction $\sqrt b$ whose own dispersion is $\nu_{\sqrt b}$. Substituting the decomposition into $\cos^2\theta_{xy}=\langle x,y\rangle^2/(\|x\|^2\|y\|^2)$, expanding $\langle x,y\rangle=\alpha\langle x,\sqrt b\rangle+\beta\|x-\pi_{\sqrt b}(x)\|^2$, and discarding the $\rho_\perp$ mass from the numerator while keeping it in $\|y\|^2$ (the worst case, controlled by $\rho$ through $\|\rho_\perp\|^2\le(1-\rho^2)\|y-\pi_{\sqrt b}(y)\|^2$) yields
\[
\cos^2\theta_{xy}\;\ge\;\cos^2\theta_{x,\sqrt b}\Big(1+\rho^2\nu_{\hat r}\,\frac{1+\nu_{\sqrt b}}{1+\nu_{\sqrt b}/\rho^2}\Big)-\delta(\lambda),
\]
where $\delta(\lambda)=O(\lambda^2\|W\|_F^2/\|Q\|_F^2)$ collects the perturbation of $y$ induced by the $\lambda W$ term inside the denominator (which shifts $\hat r_\ell$ from its $\lambda=0$ expression by a relative $O(\lambda\|W_\ell\|_F/\|s_\ell Q_\ell\|_F)$). Dividing by $\cos^2\theta_{x,\sqrt b}$ gives the claim. \qed

\subsection{Collapse and calibration}

\begin{proof}[Proof of \cref{thm:collapse}]
If $\hat r_{\ell,t}=r_{\max}$ for every $\ell$, then $y=r_{\max}\sqrt b$ is parallel to $\sqrt b$, so $\cos^2\theta_{xy}=\cos^2\theta_{x,\sqrt b}$ and $\kappa_{\mathrm{eff}}=1$ by \cref{lem:geom}: a uniform multiplier cancels in $\Phi$.
\end{proof}

Even a hypothetical re-choice of $r_{\max}$ that brings the raw-momentum ratio in-band does not restore the gain: on FineWeb-Edu 125M diagnostics the cross-layer dispersion of $\|W_\ell\|_F/\|\widetilde M_\ell\|_F$ at any fixed step is a factor ${<}2$--$3$ (its variation is dominated by \emph{time}, not layers), giving $\nu_{\hat r}\lesssim0.05$ and $\kappa_{\mathrm{eff}}\lesssim1.005$ by \cref{thm:gain}'s formula.

\begin{corollary}[Properties of the calibrated ratio]
\label{cor:cal}
With $r_{\ell,t}$ as in \cref{eq:orscale-lm}: (i) \emph{anchor}: $r_{\ell,0}=1$ to within the $\varepsilon$ floor; (ii) \emph{width-invariance}: $r_{\ell,0}=1$ holds for every $(m_\ell,n_\ell)$, and the Moonlight invariant $s_\ell\RMS(Q_{\ell,t})=0.2$ makes the per-element update RMS at the anchor $0.2\,\eta_t$ regardless of width, so the recipe LR transfers in the Moonlight+$\mu$P sense~\citep{liu2025moonlight,yang2022tensorprogramsvtuning}; (iii) \emph{ceiling}: $r_{\ell,t}\to1/(c_\ell\lambda)$ as $\|W_{\ell,t}\|_F\to\infty$ (\cref{lem:ceiling}); (iv) \emph{early-training LARS signal}: when $\lambda\|W_{\ell,t}\|_F\ll s_\ell\sqrt{p_\ell}$, $r_{\ell,t}\approx\|W_{\ell,t}\|_F/\|W_{\ell,0}\|_F$ (\cref{lem:lars}).
\end{corollary}

\begin{lemma}[Ceiling]
\label{lem:ceiling}
As $\|W_{\ell,t}\|_F\to\infty$ with $\|s_\ell Q_{\ell,t}\|_F$ bounded, $r_{\ell,t}\to1/(c_\ell\lambda)$.
\end{lemma}
\begin{proof}
$\|\lambda W+s_\ell Q\|_F=\lambda\|W\|_F(1+O(s_\ell\sqrt{p_\ell}/(\lambda\|W\|_F)))$, so the denominator is $c_\ell\lambda\|W\|_F(1+o(1))$ and the ratio tends to $1/(c_\ell\lambda)$.
\end{proof}

\begin{lemma}[Early-training regime]
\label{lem:lars}
When $\lambda\|W_{\ell,t}\|_F\ll s_\ell\sqrt{p_\ell}$, the denominator is $\approx c_\ell\,s_\ell\sqrt{p_\ell}$; substituting the leading-order $c_\ell=\|W_{\ell,0}\|_F/(s_\ell\sqrt{p_\ell})$ from \cref{def:calib} gives $r_{\ell,t}\approx\|W_{\ell,t}\|_F/\|W_{\ell,0}\|_F$.
\end{lemma}

\textbf{What the theorems do and do not say.} \emph{Do:} the rate is standard for any in-band multiplier (\cref{thm:basic}); the gain over Muon is strict under measurable conditions (\cref{thm:gain}); saturated unit-inconsistent ratios provably reduce to Muon (\cref{thm:collapse}); calibration delivers the anchor, transfer, ceiling, and LARS-signal properties (\cref{cor:cal}). \emph{Do not:} the \cref{thm:gain} bound is not claimed sharp, and (S1)--(S2) are regime conditions to be checked on diagnostics---as done in \cref{sec:moe-heterogeneity}---not universal facts.

\section{Comparison of guarantee structures}
\label{app:compare}

\begin{table}[h]
\centering
\caption{Guarantee structure on the optimization term, under each method's headline analysis. ``Width-invariant'' refers to whether the bound's leading factor is independent of $d_{\mathrm{model}}$.}
\label{tab:compare}
\footnotesize
\setlength{\tabcolsep}{3pt}
\begin{tabular}{l l l l}
\toprule
Optimizer & Smoothness factor & Layer-adaptive gain & Width-invariant \\
\midrule
SGD & $L_\infty$ & none & yes \\
LARS~\citep{you2017lars} & $L_{\mathrm{avg}}$ (Frobenius) & implicit & yes \\
LAMB~\citep{you2020lamb} & $\|L\|_1/\sqrt h$ & implicit & no ($\propto d_{\mathrm{model}}$) \\
Muon~\citep{jordan2024muon} & $\bar L_p$ (nuclear) & none ($\hat r\equiv1$) & yes \\
Mismatch-denom.\ variants & $\bar L_p$ & $=1$ under saturation (\cref{thm:collapse}) & yes \\
\textbf{OrScale} ($c_\ell{=}1$) & $\bar L_p$ & $\ge1+\rho^2\nu_{\hat r}(\cdots)>1$ (\cref{thm:gain}) & yes \\
\textbf{OrScale} (LM configuration) & $\bar L_p$ & same bound, inherited (\cref{cor:cal}) & yes (\cref{cor:cal}(ii)) \\
\bottomrule
\end{tabular}
\end{table}

Relative to the Muon convergence literature (\cref{sec:related}): Gluon-style analyses~\citep{riabinin2025gluon} obtain layer-wise stepsizes from $(L^0_\ell,L^1_\ell)$ smoothness constants, which are sharper but unobservable; \cref{thm:gain} is deliberately conditioned on statistics of the logged $\hat r_{\ell,t}$, trading tightness for a condition that any practitioner can check on their own run.

\section{Design-choice ablation details}
\label{app:req-details}

This appendix expands the symptoms summarized in \cref{tab:req-ablation}; variant names in the released code are noted for traceability. In the released code, configs, and logs, the LM-configuration arms of the main experiments carry the key \texttt{orscale\_lm}.

\textbf{Post-orthogonalization denominator (omitting applied-direction norm)} (code: \texttt{orscale\_muon}/FM1 lineage). With $Q_\ell$ the polar factor, $\|Q_\ell\|_F=\sqrt{\rank Q_\ell}\approx\sqrt{p_\ell}$ (\cref{lem:polar}), so $\|W_\ell\|_F/\|Q_\ell\|_F$ depends on weight norm and shape only. The variant can still score on CIFAR---weight-RMS variation supplies a residual per-layer signal---but the ratio carries no information about the applied update, and any adaptive interpretation is unfounded.

\textbf{Raw-momentum denominator (omitting applied-direction norm)} (code: \texttt{mutrust}, \texttt{muscale}). The numerator has weight units; $\|\widetilde M_\ell\|_F$ has gradient units and is orders of magnitude smaller. On FineWeb-Edu 125M with $(r_{\min},r_{\max})=(0.5,1.5)$, the raw cross-layer mean of $\|W\|_F/\|\widetilde M\|_F$ spans $24.5$--$19{,}128$ across the LR grid, and the clipped ratio sits at $r_{\max}$ on $99.75$--$100\%$ (Frobenius form) and $99.5$--$100\%$ (RMS/Moonlight form) of diagnostic steps. \Cref{thm:collapse} then applies: the multiplier is uniform, $\kappa_{\mathrm{eff}}=1$, and the method is Muon with a relabeled learning rate---which is exactly how it scores.

\textbf{Uncalibrated ratio (omitting calibration anchor)} (code: \texttt{orscale\_muon\_moonlight}). \Cref{sec:calibration-ablation}; the uncorrected anchor $r_{\ell,0}\propto1/\sqrt{\mathrm{fan\mbox{-}in}}\in[0.11,0.26]$ across our 125M layers.

\textbf{Decoupled weight decay (omitting WD coupling)} (code: FM3 lineage). With the calibrated denominator but a decoupled pull-back $(1-\eta\lambda)W$, FineWeb-Edu 125M diagnostics show $10$--$50\times$ growth of $\|W_\ell\|_F$ on \texttt{mlp.down}, \texttt{mlp.up}, and \texttt{attn.out\_proj} before the ratio pins at $r_{\max}$. The mechanism is the positive feedback loop of \cref{sec:derivation}: the pull-back does not scale with $r_\ell$, so nothing damps growth until the clip does; coupling the decay into the scaled step creates the finite ceiling of \cref{lem:ceiling} instead.

\textbf{Clip removal.} With $(r_{\min},r_{\max})\to(-\infty,\infty)$ at 125M, the largest transient $\hat r$ reaches $9.6$ during the first ${\sim}5\%$ of steps before the coupled-WD ceiling engages; final CE $3.2187$ vs.\ $3.2120$ with the default clip. At 16B, doubling $r_{\max}$ from $5$ to $10$ shifts final CE by $+0.0146$ (\cref{tab:moe}).

\section{Full hyperparameters}
\label{app:hyperparams}

\textbf{Dense LM (FineWeb-Edu).} GPT-2-style decoder-only transformer: pre-norm, RMSNorm, rotary position embeddings, ReLU$^2$ MLP, no biases, tied input/output embeddings, BPE vocabulary 50{,}304 (GPT-2 tokenizer). Precision bf16; gradient accumulation to the global batches below; cosine LR schedule with 1{,}000 warmup steps; matrix parameters use the optimizer under test, non-matrix parameters the AdamW fallback at the same LR with betas $(0.9,0.95)$ and weight decay $0.1$ (matching Moonlight's dense-baseline settings). OrScale (LM configuration): $r_{\min}=0.1$, $r_{\max}=5.0$, $\mu=0.95$, $\lambda=0.1$, $k=5$, $\varepsilon=10^{-6}$, fp32 calibration. Per-scale geometry and shared recipe LRs:

\begin{table}[h]
\centering
\footnotesize
\caption{Dense-LM scales. ``Batch'' is examples per optimizer step (global); tokens/step $=$ batch $\times$ seq. At 545M+ the sequence length is shortened from Moonlight's 8192 to 4096 with batch examples doubled, keeping tokens/step, steps, and the LR schedule unchanged; every optimizer at a given scale sees the identical setting, so within-scale comparisons are unaffected (absolute losses at 545M+ are not comparable to 8K-context numbers).}
\label{tab:dense-hp}
\begin{tabular}{l c c c c c c}
\toprule
Scale & Seq len & Batch (ex.) & Tokens & Steps & Recipe LR & Warmup \\
\midrule
125M & 1024 & 256 & 5.24B & 20{,}000 & $1.0\times10^{-3}$ & 1{,}000 \\
399M & 8192 & 96 & 8.92B & 11{,}342 & $9.503\times10^{-4}$ & 1{,}000 \\
545M & 4096 & 256 & 14.04B & 13{,}390 & $9.143\times10^{-4}$ & 1{,}000 \\
1.1B & 4096 & 384 & 28.54B & 18{,}146 & $8.561\times10^{-4}$ & 1{,}000 \\
\bottomrule
\end{tabular}
\end{table}

\textbf{16B-A3B MoE.} Moonlight-16B-A3B configuration: 27 layers, hidden size 2048, 16 heads, multi-head latent attention (KV LoRA rank 512, $q_{\mathrm{nope}}$ 128, $q_{\mathrm{rope}}$ 64, $v$ head dim 128), 64 routed experts with top-6 sigmoid routing and routed scaling factor 2.446, 2 shared experts, MoE intermediate 1408, dense intermediate 10944 (first layer dense), aux-loss-free load balancing via router-bias updates (rate $10^{-3}$, aux coefficient $0$), vocabulary 163{,}840, sequence length 8192. Total 15{,}958.1M parameters, 2{,}241.7M active per token. Training: 8$\times$H20, bf16, ZeRO-1 optimizer-state sharding, micro-batch 1 with gradient accumulation 256 per rank ($=2048$ examples $=16{,}777{,}216$ tokens/step), 596 steps $=9{,}999{,}220{,}736$ FineWeb-Edu tokens (GPT-2 tokenizer), warmup 60 steps, cosine schedule, seed 42, validation every 50 steps over 10 batches. All optimizers: LR $4.2\times10^{-4}$, weight decay $0.1$; Muon-family: $\mu=0.95$, $k=5$; AdamW (standalone and fallback groups): betas $(0.9,0.95)$; OrScale: $r_{\min}=0.1$, $r_{\max}\in\{5,10\}$.

\textbf{CIFAR-10 / DavidNet.} 24 epochs, global batch 512, bf16, cosine schedule with 500-step warmup, standard augmentation, label smoothing 0, no gradient clipping; three seeds per cell; per-family LR grids with weight decay $5\times10^{-4}$; OrScale clip $(0.5,1.5)$; AdamW fallback LR $10^{-3}$, betas $(0.9,0.999)$. Val top-1 averages the last three epochs, then seeds.

\textbf{Uncalibrated variant LRs (\cref{tab:calibration}).} The uncalibrated variant's matrix LR was tuned to $3\times10^{-3}$ at 125M (grid evidence) and $2.376\times10^{-3}$ at 399M ($2.5\times$ recipe, the rule extrapolated from the 125M ratio); its AdamW group stays at the recipe LR.

\section{16B-A3B additional results}
\label{app:moe-extra}

\begin{figure}[h]
\centering
\includegraphics[width=0.93\linewidth]{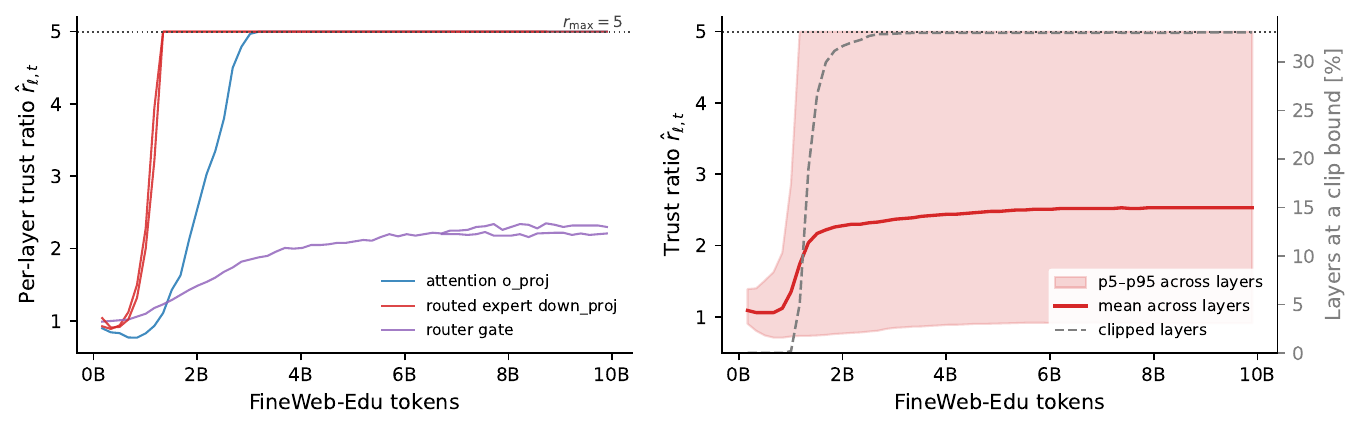}
\caption{Logged OrScale trust ratios at 16B-A3B ($r_{\max}{=}5$ arm), behind the heterogeneity check of \cref{sec:moe}. \emph{Left:} per-layer trajectories $\hat r_{\ell,t}$ by layer class: routed-expert projections rise to the ceiling within 1.5B tokens, attention output projections follow by 3B, router gates settle near $2.2$--$2.3$ and never clip. \emph{Right:} mean and p5--p95 band across all matrix layers; the fraction of layers at a clip bound (gray, right axis) reaches $33\%$ late in training.}
\label{fig:moe-trust}
\end{figure}

\textbf{Trust-ratio heterogeneity in full.} \Cref{fig:moe-trust} shows the logged ratios of the $r_{\max}{=}5$ arm, expanding the heterogeneity check of \cref{sec:moe}. The calibration constants themselves span an order of magnitude across layers ($c_\ell$ p5--p95: $0.0047$--$0.0446$), so no static rescaling could reproduce the layer-class structure. The $33\%$ late-training clip fraction is \emph{not} the saturation pathology of \cref{thm:collapse}---that failure pins \emph{every} layer at $r_{\max}$ from the first steps---but the coupled-WD ceiling engaging for the layer class that wants the largest relative steps while other classes remain in-band; consistently, loosening the ceiling to $r_{\max}{=}10$ moves the final loss by only $+0.015$ nats (\cref{fig:moe-trust-rmax10}). Under the dense 125M diagnostics the same statistics are present but weaker ($\nu_{\mathrm{growth}}\approx 0.4$), matching the smaller dense-regime gains: the theory's heterogeneity dial and the empirical gap move together.

\begin{figure}[h]
\centering
\includegraphics[width=0.85\linewidth]{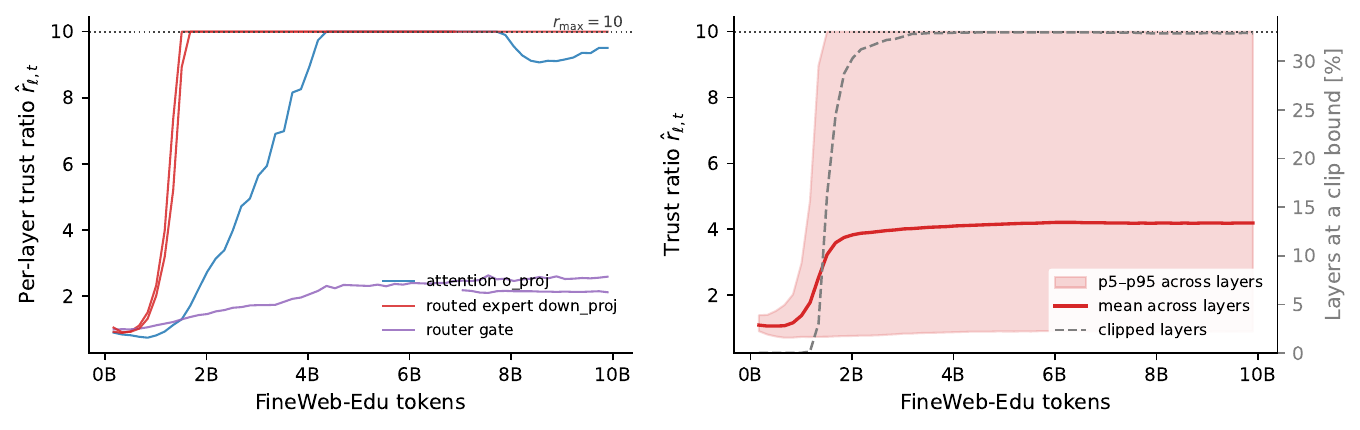}
\caption{Per-layer trust ratios and layer-aggregate band for the $r_{\max}{=}10$ arm. The layer-class ordering (experts $>$ attention $>$ router) is identical to \cref{fig:moe-trust}; expert projections settle near the looser ceiling, and the run finishes $0.0146$ nats behind $r_{\max}{=}5$.}
\label{fig:moe-trust-rmax10}
\end{figure}

\begin{figure}[h]
\centering
\includegraphics[width=0.6\linewidth]{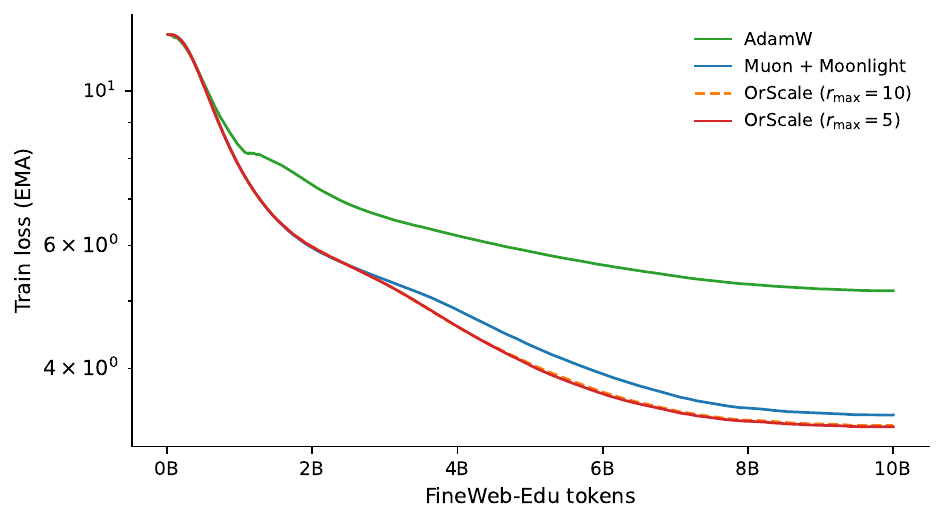}
\caption{Smoothed training loss (EMA) for all four 16B-A3B arms.}
\label{fig:moe-train}
\end{figure}

Throughput: AdamW 28.7k tok/s/rank-group, Muon+Moonlight 23.2k, OrScale 23.5k ($r_{\max}{=}5$) and 23.4k ($r_{\max}{=}10$); the trust-ratio computation is within measurement noise of Muon+Moonlight, and the AdamW-vs-Muon-family difference is the Newton--Schulz cost, unrelated to OrScale. \Cref{fig:moe-train} overlays smoothed training loss for all four arms; validation curves for all arms at all 12 checkpoints are included in the supplementary CSVs.

\section{Dense-regime additional results}
\label{app:lr-sweep}

\textbf{Matched 125M LR sweep.} Both optimizers on the same four-point grid, single seed:

\begin{table}[h]
\centering
\footnotesize
\caption{125M FineWeb-Edu final validation CE across a shared LR grid. OrScale is at or below Muon+Moonlight at every stable cell; both optima sit at the recipe LR $10^{-3}$; Muon+Moonlight diverges at $10^{-2}$ while OrScale still trains.}
\label{tab:lr-sweep}
\begin{tabular}{l c c c c}
\toprule
Optimizer & $3\times10^{-4}$ & $10^{-3}$ & $3\times10^{-3}$ & $10^{-2}$ \\
\midrule
Muon+Moonlight & $3.2841$ & $\mathbf{3.2319}$ & $3.2372$ & $4.6262$ (diverged) \\
OrScale & $3.2588$ & $\mathbf{3.2120}$ & $3.2196$ & $3.3287$ \\
\bottomrule
\end{tabular}
\end{table}

\textbf{Scaling-law fits.} Kaplan-style log--log fits over the four dense scales give $L=2.826\,C^{-0.054}$ (AdamW), $2.731\,C^{-0.053}$ (Muon+Moonlight), $2.725\,C^{-0.052}$ (OrScale); the AdamW and Muon exponents match Moonlight's published fits to within $0.002$, an internal sanity check on the setup. With four points per optimizer we treat the per-scale deltas of \cref{tab:dense} as primary and the fitted exponents as descriptive; \cref{fig:dense-cifar} (left, main text) plots the per-scale gap.

\section{CIFAR-10 details}
\label{app:cifar}

\begin{figure}[h]
\centering
\includegraphics[width=0.62\linewidth]{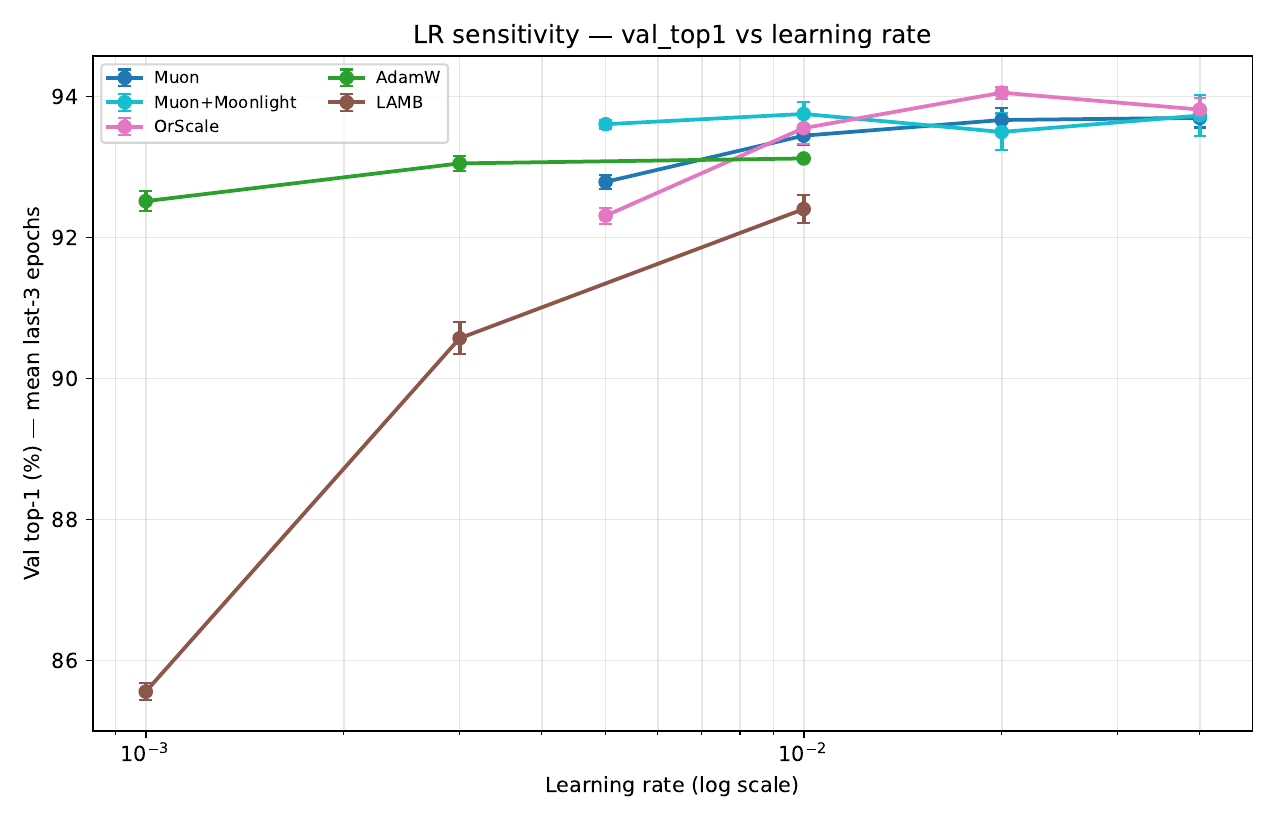}
\caption{CIFAR-10 / DavidNet validation top-1 vs.\ learning rate, five main baselines, three seeds per cell.}
\label{fig:cifar-lr}
\end{figure}

\begin{table}[h]
\centering
\footnotesize
\caption{CIFAR-10 / DavidNet leaderboard at each optimizer's best LR (three seeds, $\pm1\sigma$).}
\label{tab:cifar}
\begin{tabular}{r l l c}
\toprule
Rank & Optimizer & LR & Val top-1 (\%) \\
\midrule
1 & \textbf{OrScale} (ours) & 0.02 & $\mathbf{94.05\pm0.08}$ \\
2 & Muon + Moonlight & 0.01 & $93.75\pm0.17$ \\
3 & Muon & 0.04 & $93.70\pm0.14$ \\
4 & AdamW & 0.01 & $93.12\pm0.04$ \\
5 & LAMB & 0.01 & $92.40\pm0.20$ \\
\bottomrule
\end{tabular}
\end{table}

Two notes on the LR-sensitivity comparison. First, nominal LR grids are not directly comparable across families: Moonlight's shape factor multiplies the effective Muon step by $0.2\sqrt{\max(m,n)}\approx4.8$--$13.6$ for DavidNet's convolutions, so Muon+Moonlight's nominal grid sits at a Muon-effective position several times higher (empirically, Moonlight at LR $0.005$ matches Muon at $0.02$: $93.61$ vs.\ $93.67$). Second, extending the grid downward, Muon+Moonlight at nominal $\{5\times10^{-4},10^{-3}\}$ falls to $92.41\pm0.12$ and $93.18\pm0.09$, off its $93.75$ plateau---its apparent flatness on the nominal grid partly reflects grid placement. At OrScale's best LR the gain over the best Muon-family baseline is $+0.30$ points with a Welch $t\approx2.8$ across seeds.

\section{Compute resources}
\label{app:compute}

CIFAR-10/DavidNet: ${\sim}2$ minutes per run on a single RTX-class GPU; the full sweep ${\sim}8$ GPU-hours. Dense LM: 8$\times$H20-3E, from ${\sim}3$ hours (125M) to ${\sim}3$--$5$ days (1.1B) per run. 16B-A3B MoE: $3.9$--$5.0$ days per arm on 8$\times$H20 ($4$ arms; wall-clock in \cref{tab:moe}). Total project compute including preliminary and failed sweeps: ${\sim}5{,}000$ GPU-hours (the four 16B-A3B arms alone account for ${\sim}3{,}600$).

\section{Broader impact}
\label{app:impact}

OrScale is a pre-training-efficiency modification: its direct effect is reducing compute and energy per unit of pre-training progress. Indirect negative effects mirror those of the systems it accelerates (cheaper pre-training lowers the cost of downstream misuse, bias, and environmental externalities); it introduces no new dataset, surveillance, or model-release pathway beyond the optimizer implementation.

\end{document}